\documentclass{amsart}
\title{Convergence of Expected Utilities with Algorithmic Probability Distributions}
\author{Peter de Blanc\\Department of Mathematics\\Temple University}
\thanks{Thanks to Nick Hay at Cornell University for help with the Recursion Theorem.}
\date{\today}
\begin{document}
\begin{titlepage}
\maketitle
\end{titlepage}
\section{Abstract}

We consider an agent interacting with an unknown environment. The environment is a function which maps natural numbers to natural numbers; the agent's set of hypotheses about the environment contains all such functions which are computable and compatible with a finite set of known input-output pairs, and the agent assigns a positive probability to each such hypothesis (Probability distributions over all computable functions are used in theoretical AI systems such as AIXI (Hutter, 2007)). We do not require that this probability distribution be computable, but it must be bounded below by a positive computable function.

The agent has a utility function on outputs from the environment. We show that if this utility function is bounded below in absolute value by an unbounded computable function, then the expected utility of any input is undefined.

This implies that a computable utility function will have convergent expected utilities iff that function is bounded.

\section{Notation}

Here we set up our notation.

Let $R$ be the set of partial $\mu$-recursive functions, and let $S = R \cap \mathbb{N}^\mathbb{N}$; that is, $S$ is the set of total $\mu$-recursive functions on a single argument. We will use an index $R = \{\phi_n\}_{n=0}^{\infty}$, where our index $\phi$ is a G\"{o}del numbering of the computable functions.

Let $h \in \mathbb{N}^{\mathbb{N}}$ be the true function that describes the environment. The agent has some finite set of tested inputs $I$, so the agent knows the value of $h(i)$ for all $i \in I$. Let $S_I = \{f \in S : (\forall i \in I), f(i) = h(i)\}$. That is, $S_I$ is the set of hypotheses which agree with the agent's knowledge of tested inputs.

Let $H$ be our set of hypotheses about the environment, with $S_I \subseteq H \subseteq \mathbb{N}^{\mathbb{N}}$.

Let $p: H \rightarrow \mathbb{R}$ be our probability distribution on the hypothesis set. We require that there exist some computable function $\bar{p} : \mathbb{N} \rightarrow \mathbb{Q}$ such that $(\forall \phi_n \in S_I), 0 < \bar{p}(n) \le p(\phi_n)$.

Let $U : \mathbb{N} \rightarrow \mathbb{R}$ be the agent's utility function. We suppose that $U$ is unbounded, and that there exists an unbounded computable function $\bar{U}: \mathbb{N} \rightarrow \mathbb{Q} : (\forall n \in \mathbb{N}), |\bar{U}(n)| \le |U(n)|$.

\section{Proof of Divergence of Expected Utilities}
\label{thesec}

Fix $k \in (\mathbb{N} - I)$, an input whose expected utility we will consider. Then define $B : \mathbb{N} \rightarrow \mathbb{N}$, with

\begin{equation}
	\label{bdef}
	(\forall x \in \mathbb{N}), B(x) = \max{\{\phi_n(k) : n \in \mathbb{N}, n \le x\}}
\end{equation}

The sequence $\left\{B(j)\right\}_{j=0}^{\infty}$ can be thought of as an analog of the Busy Beaver numbers (Rad\'{o}, 1962). This sequence will be used in proving our main result. Note that $B$ is not a computable function, as we are about to show.

\newtheorem{lemb}{Lemma}
\begin{lemb}
	\label{thelemb}
	Let $f \in S$. Then $B(x) > f(x)$ infinitely often.
\end{lemb}
\begin{proof}
	Suppose not. Then $B(x) > f(x)$ only finitely many times, so let $F(x) = 1 + f(x) + \max{\{B(x) - f(x) : x \in \mathbb{N}\}}$. Then $F \in S$, and $(\forall x \in \mathbb{N}), F(x) > B(x)$.

	Let $Q: \mathbb{N}^2 \rightarrow \mathbb{N}$, $Q(i, x) = F(i)$. By a corollary of the recursion theorem (Kleene, 1938), there exists $p \in \mathbb{N}$ such that $(\forall x \in \mathbb{N}), \phi_p(x) = Q(p, x) = F(p)$.

	By equation ~\ref{bdef}, $B(p) \ge \phi_p(k) = F(p)$. This contradicts our statement that $(\forall x \in \mathbb{N}), F(x) > B(x)$.
\end{proof}

The expected utility of inputting $k$ to the environment is:

\begin{equation}
	\label{euseries}
	EU(h(k)) = \sum_{f \in H} p(f) U(f(k))
\end{equation}

\newtheorem{bigthm}{Theorem}
\begin{bigthm}
	The series in equation ~\ref{euseries} does not converge.
\end{bigthm}
\begin{proof}
To establish that this series does not converge, we will show that infinitely many of its terms have absolute value $\ge 1$. We will do this by constructing a sequence of hypotheses in $S_I$ whose utility grows very quickly - as quickly as the function $B$ - faster than their probabilities can shrink.

By equation ~\ref{bdef}, for each $j \in \mathbb{N}$ there exists $u_j \in \mathbb{N}$, $u_j \le j$ such that $\phi_{u_j}(k) = B(j)$. Now we define a map on function indices $G: \mathbb{N} \rightarrow \mathbb{N}$ such that:
\begin{displaymath}
	\phi_{G(n)}(x) = \left\{ \begin{array}{ll}
		h(x) & \textrm{if $x \in I$}\\
		\min{\{y \in \mathbb{N} : |\bar{U}(y)| \ge |\phi_n(k)|\}} & \textrm{otherwise}
	\end{array} \right.
\end{displaymath}
Because $|\bar{U}|$ is an unbounded function, $\phi_{G(n)}$ is a total function as long as $\phi_n(k)$ is defined. By definition, $(\forall n \in \mathbb{N}), \phi_{G(n)} \in S_I \subseteq H$.

So our sequence of hypotheses will be $\left\{ \phi_{G(u_j)} \right\}_{j=1}^{\infty}$.  Because $k \notin I$, we have:
\begin{equation}
	\label{utileq}
	(\forall j \in \mathbb{N}), |U(\phi_{G(u_j)}(k))| \ge |\bar{U}(\phi_{G(u_j)}(k))| \ge \phi_{u_j}(k) = B(j)
\end{equation}

We're almost done. We now let $q : \mathbb{N} \rightarrow \mathbb{N}$, with:

\begin{equation}
	\label{qdef}
	q(x) = \lceil \sup{\{\frac{1}{\bar{p}(G(y))} : y \in \mathbb{N}, y \le x\}} \rceil
\end{equation}

Then $q$ is a nondecreasing function in $S$ with the property that $(\forall x \in \mathbb{N}), q(x) \ge \bar{p}(G(x))^{-1}$. By Lemma ~\ref{thelemb}, $B(x) \ge q(x)$ infinitely often. It follows from equations ~\ref{utileq} and ~\ref{qdef} that:
\begin{equation}
	\label{bigineq}
	|U(\phi_{G(u_j)}(k))| \ge B(j) \ge q(j) \ge q(u_j) \ge \bar{p}(G(u_j))^{-1} \ge p(\phi_{G(u_j}))^{-1}
\end{equation}
for infinitely many $j \in \mathbb{N}$.

Because $p(G(u_j))$ is always positive, and by the above equation, it follows that:
\begin{equation}
	|p(\phi_{G(u_j})) U(\phi_{G(u_j)}(k))| \ge 1
\end{equation} 
infinitely often. Because these are all terms of the series in equation ~\ref{euseries}, the series can not converge.
\end{proof}

\section{Remarks}
We have shown that utility functions which are bounded below in absolute value by an unbounded computable function do not have convergent expected utilities.

Because all computable functions are bounded below in absolute value by themselves, it follows that all unbounded computable utility functions have divergent expected utilities. Since bounded utility functions always have convergent expected utilities, we know that \emph{a computable utility function has convergent expected utilities iff it is bounded.}

\end{document}